\documentclass[11pt,oneside]{article}

\usepackage{palatino}
\usepackage[margin=1in]{geometry}
\usepackage{color,xcolor,soul}
\usepackage[abspage,user,savepos]{zref}
\usepackage[colorlinks, citecolor=blue]{hyperref}  
\usepackage{natbib}
\usepackage{fp}
\usepackage{framed}
\usepackage{cuted,balance}
\usepackage{setspace}
\usepackage{subcaption}
\usepackage{dingbat}
\usepackage{verbatim}
\usepackage{enumerate}
\usepackage{graphicx} 
\usepackage{epsfig} 
\usepackage{mathtools}
\usepackage{dsfont,amsmath} 
\usepackage{amssymb}  
\usepackage{authblk}
\usepackage{algorithm,algorithmic}
\usepackage{xcolor}
\usepackage{wasysym}
\usepackage{tikz}
\usetikzlibrary{shapes}
\usepackage{hyperref,cleveref}       
\usepackage{overpic}
\usepackage{amsopn}

\newcounter{definition}
\newenvironment{definition}[1]{\refstepcounter{definition}\par\medskip
\noindent 
\textbf{Definition \thedefinition~(#1)} \em \rmfamily}
{\medskip}
\newenvironment{definition*}{\refstepcounter{definition}\par\medskip
\noindent 
\textbf{Definition \thedefinition.} \em \rmfamily}
{\medskip}

\newcounter{theorem}
\newenvironment{theorem}[1]{\refstepcounter{theorem}\par\medskip
\noindent 
\textbf{Theorem \thetheorem~(#1)} \em \rmfamily}
{\medskip}
\newenvironment{theorem*}{\refstepcounter{theorem}\par\medskip
\noindent 
\textbf{Theorem \thetheorem.} \em \rmfamily}
{\medskip}

\newcounter{assumption}
\newenvironment{assumption}[1]{\refstepcounter{assumption}\par\medskip
\noindent 
\textbf{Assumption \theassumption~(#1)} \em \rmfamily}
{\medskip}
\newenvironment{assumption*}{\refstepcounter{assumption}\par\medskip
\noindent 
\textbf{Assumption \theassumption.} \em \rmfamily}
{\medskip}

\newcounter{proposition}
{\medskip}
\newenvironment{proposition*}{\refstepcounter{proposition}\par\medskip
\noindent 
\textbf{Proposition \theproposition.} \em \rmfamily}
{\medskip}

\newenvironment{proof}{
\indent \textit{Proof.} \rmfamily}{\hfill $\square$}

\crefname{assumption}{Assumption}{Assumptions}

\DeclareMathOperator*{\argmax}{arg\,max}

\newcommand{\Markov}{\mathcal{M}}
\newcommand{\Markoveps}{\mathcal{M}_\epsilon}
\newcommand{\State}{\mathcal{S}}
\newcommand{\Action}{\mathcal{A}}

\newcommand{\Reward}{\mathsf{R}}
\newcommand{\Rewardeps}{\mathsf{R}_\epsilon}
\newcommand{\peps}{p_\epsilon}
\newcommand{\opeps}{\overline{p}_\epsilon}
\newcommand{\reward}{\overline{r}}
\newcommand{\rewardeps}{\reward_\epsilon}
\newcommand{\explore}{\mathcal{E}}
\newcommand{\Ueps}{U_\epsilon}

\newcommand{\Real}{\mathbb{R}}
\newcommand{\E}{\mathrm{E}}
\newcommand{\ind}[1]{\mathbb{I}_{\{#1\}}}
\newcommand{\BR}{\mathrm{br}_{\epsilon}}  
\newcommand{\BRfree}{\mathrm{br}}

\newcommand{\hatv}{\hat{v}}
\newcommand{\bary}{\overline{y}}
\newcommand{\Gap}{\mathrm{NG}}

\newcommand{\uY}{x}
\newcommand{\oY}{\overline{x}}

\newcommand{\oDelta}{\overline{\Delta}}
\newcommand{\uDelta}{\underline{\Delta}}

\newcommand{\hatpi}{\hat{\pi}}
\newcommand{\hatq}{\hat{q}}
\newcommand{\hatQ}{\hat{Q}}

\newcommand{\tlambda}{\tilde{\lambda}}

\newcommand{\be}{\begin{equation}}
\newcommand{\ee}{\end{equation}}
\newcommand{\nn}{\nonumber}

\newcommand{\vvec}{\mathbf{v}}
\newcommand{\Pvec}{\mathbf{P}}
\newcommand{\rvec}{\mathbf{r}}

\begin{document}

\date{}

\title{\LARGE \bf
Actor-Dual-Critic Dynamics for Zero-sum and Identical-Interest Stochastic Games}

\author{Ahmed Said Donmez, Yuksel Arslantas and Muhammed O. Sayin
\thanks{A. S. Donmez, Y. Arslantas and M. O. Sayin are with the Department of Electrical \& Electronics Engineering at Bilkent University, Ankara, T\"{u}rkiye 06800. 
(Emails:  { \tt\small said.donmez@bilkent.edu.tr}, {\tt\small yuksel.arslantas@bilkent.edu.tr},  { \tt\small sayin@ee.bilkent.edu.tr})}%
}

\maketitle

\bigskip

\begin{center}
\textbf{Abstract}
\end{center}

We propose a novel independent and payoff-based learning framework for stochastic games that is model-free, game-agnostic, and gradient-free. The learning dynamics follow a best-response-type actor-critic architecture, where agents update their strategies (actors) using feedback from \textit{two} distinct critics: a fast critic that intuitively responds to observed payoffs under limited information, and a slow critic that deliberatively approximates the solution to the underlying dynamic programming problem. Crucially, the learning process relies on non-equilibrium adaptation through smoothed best responses to observed payoffs. We establish convergence to (approximate) equilibria in two-agent zero-sum and multi-agent identical-interest stochastic games over an infinite horizon. This provides one of the first payoff-based and fully decentralized learning algorithms with theoretical guarantees in both settings. Empirical results further validate the robustness and effectiveness of the proposed approach across both classes of games.

\begin{spacing}{1.245}

\section{Introduction}

Reinforcement learning (RL) has achieved remarkable success in domains such as board games \citep{ref:Silver17}, autonomous driving \citep{ref:Shwartz16}, city-scale navigation \citep{ref:Mirowski18}, and robotic control \citep{ref:Gu23}. Many of these applications inherently involve multiple interacting agents, underscoring the importance of \emph{multi-agent reinforcement learning} (MARL) as a framework for sequential decision-making in shared environments \citep{ref:Zhang21a}. Despite its empirical progress, MARL still relies heavily on heuristics and parameter tuning, often lacking strong theoretical guarantees.

Game theory provides a principled foundation for analyzing strategic interactions among intelligent agents. When combined with sequential decision-making, it gives rise to \emph{stochastic games} (or Markov games), a natural and expressive model for MARL \citep{ref:Shapley53,ref:Zhang21a}. However, stochastic games inherit and amplify the distinct challenges of both reinforcement learning and multi-agent interaction. Properties that hold in single-agent MDPs or in repeated matrix games may fail to extend to stochastic games, complicating both analysis and algorithm design. Recent work has sought to overcome these challenges and provide theoretical insight into algorithmic learning in stochastic games \citep{ref:Daskalakis20,ref:Sayin21,ref:Daskalakis23,ref:Ding22,ref:Jin22}. 

\subsection{Contributions}
In this paper, we focus on decentralized learning settings where agents update independently, without access to opponent actions, or environment models. Such minimal-information scenarios are highly relevant in practice yet challenging. Our goal is to develop learning algorithms that operate under these constraints while retaining convergence guarantees. This leads us to the question:
\begin{center}
\emph{Can independent agents, relying solely on payoff feedback and non-equilibrium adaptation,
converge to equilibrium in key classes of stochastic games?}
\end{center}
We provide an \textit{affirmative} answer for two benchmark cases: zero-sum and identical-interest stochastic games. 

To address the challenge of decentralized learning in stochastic games, we introduce a cognitively inspired, payoff-based actor-critic framework. Our key contributions are as follows:

\textit{A novel actor-dual-critic architecture:} We propose a new learning design where each agent independently maintains two critics: a \emph{fast critic}, which rapidly estimates local rewards and transitions based on local observations, and a \emph{slow critic}, which solves a fixed-point equation to approximate long-term values. Motivated by dual-process theories of human cognition \citep{ref:Kahneman11}, blending fast, reactive intuition with slow, deliberate planning, this dual decomposition addresses the challenge of non-stationarity in MARL: as each agent’s strategy evolves, it alters the environment faced by others. Treating this environment as stationary, as in single-agent RL, can lead to instability. Our architecture separates short-term reactions from long-term stabilization, allowing for more robust adaptation to such non-stationary issues.

\textit{Gradient-free, best-response-type strategy updates:} Our actor improves its strategy using \textit{best-response-type} updates rather than relying on policy gradients. This avoids the limitations of gradient-based actor-critic methods \citep{ref:Konda99}, which often converge only to local optima. Our approach is \textit{gradient-free}, enabling simple and efficient updates, \textit{game-agnostic}, avoiding assumptions about structure, and \textit{provably convergent} to equilibria in structured game classes.
Furthermore, it aligns with theoretical arguments that equilibrium behavior can emerge from non-equilibrium adaptation processes \citep{ref:Fudenberg09}. Our analysis demonstrates that even under minimal information, such payoff-driven best-response-type dynamics can robustly converge to equilibrium—thereby reinforcing the predictive power of equilibrium in stochastic games.

\textit{Convergence to equilibrium in stochastic games:} We prove convergence to approximate Nash equilibria in two foundational settings: two-agent zero-sum stochastic games, and multi-agent identical-interest stochastic games. The approximation error is proportional to the exploration rate $\epsilon$. Convergence in identical-interest games is especially challenging. Unlike the zero-sum case, these games may lack unique game values for different equilibria, and value iteration may not have the contraction property. Our analysis leverages and extends recent advances in \textit{quasi-monotonicity} techniques \citep{ref:Baudin22,ref:Sayin22b} to ensure convergence in both settings under fully independent, payoff-based dynamics.


\subsection{Related Work}
Algorithmic learning in stochastic games has recently received significant attention due to their applications on MARL \citep{ref:Ozdaglar21}. For example, \cite{ref:Leslie20} presented a continuous-time generalization of fictitious play and best response dynamics to two-agent zero-sum stochastic games. \cite{ref:Sayin22} introduced uncoupled discrete-time fictitious-play dynamics in which each agent independently estimates continuation values based on their own beliefs. \cite{ref:Baudin22} showed the convergence of \cite{ref:Leslie20}'s dynamics also for identical-interest stochastic games. \cite{ref:Sayin22b} showed the convergence of \cite{ref:Sayin22}'s dynamics also for identical-interest stochastic games with single controllers. Notably, \cite{ref:Sayin21} presented independent and payoff-based dynamics generalizing \cite{ref:Leslie05}'s individual Q-learning to stochastic games with convergence guarantees for zero-sum games. \cite{ref:Chen23} presented another independent and payoff-based dynamics generalizing \cite{ref:Leslie06}'s actor-critic dynamics to stochastic games with non-asymptotic convergence guarantees for zero-sum games. 

Several decentralized actor-critic methods have also been proposed under structural assumptions. 
\cite{ref:Arslan17} introduced decentralized Q-learning for pure-strategy equilibria, transforming stochastic games into a strategic-form game over Markov stationary deterministic strategies. While convergence is shown for weakly acyclic games, the restriction to pure strategies limits generality—particularly in zero-sum games where no pure-strategy equilibrium may exist. \cite{ref:Perolat18} developed actor-critic learning for a class of multi-stage games, where each state is visited only once. In the zero-sum setting, \cite{ref:Daskalakis20} studied independent policy gradient methods with two-timescale step sizes across agents, achieving non-asymptotic guarantees. Our actor-dual-critic dynamics differ from these works with provable convergence in both two-agent zero-sum and multi-agent identical-interest stochastic games with no access to opponent actions, game model, and dynamics. 

On the other hand, for Markov potential games, which generalize identical-interest games, convergence of independent \textit{policy gradient dynamics} has been established in several works \citep{ref:Leonardos21,ref:ZhangR24,ref:Fox22,ref:Ding22}.  By contrast, general—especially zero-sum—stochastic games lack such a Lyapunov structure, so researchers resort to coupled-Lyapunov or extra-gradient updates that still require cross-agent sampling coordination \citep{ref:Wei21, ref:Chen22}. Finite-time convergence is achieved by optimistic gradient descent-ascent and related extra-gradient variants, yet only under strong assumptions like state-space irreducibility or double-timescale updates that limit practical deployment \citep{ref:Wei21, ref:Zhao22}. Notably, \cite{ref:Ding22} shows the convergence of independent policy gradient for both zero-sum stochastic games and Markov potential games. 

The literature on learning in general-sum stochastic games mostly focus on sample efficient algorithms with regret guarantees that finds coarse correlated equilibria \citep{ref:Jin22,ref:Liu21,ref:Bai20,ref:Song21,ref:Mao23,ref:Tian21} since finding a Nash equilibria is a PPAD-complete task \citep{ref:Daskalakis23}. \cite{ref:Cai23} showed a weaker concept of convergence, path convergence, on general stochastic games. On the other hand, \cite{ref:Maheshwari24} introduced Markov Near-Potential Function idea to approximate general-sum games with potential functions and shows certain bounds on closeness to Nash equilibria. \cite{ref:Yongacoglu23} showed convergence on symmetric stochastic games. With more structural constraints, \cite{ref:Park23} demonstrated convergence on multi-agent zero-sum stochastic games with networked separable interactions, which is a generalization of polymatrix games to the stochastic game setting.

The rest of the paper is organized as follows. We describe stochastic games in \Cref{sec:Markov} and introduce the actor-dual-critic dynamics in \Cref{sec:ac}. We present the convergence results and proofs for both zero-sum and identical-interest stochastic games in \Cref{sec:results} and \Cref{sec:proof}, respectively. We provide illustrative examples in \Cref{sec:examples} and conclude the paper with some remarks in \Cref{sec:conclusion}. 

\textbf{Notations.} Let $[n]:=\{1,\ldots,n\}$ be the index set of agents and superscripts denote agent identity. Denote agent indices other than $i$ by $-i:=\{j\}_{j\neq i}$. Let $\Delta(\Action)$ denote the probability simplex over a finite set $\Action$. Denote the vector of ones by $\mathbf{1}$. Given a matrix $A$, $A^T$ denotes its transpose.

\section{Stochastic Games}\label{sec:Markov}

We consider an $n$-agent discounted stochastic game $\Markov=(\State,(\Action^i)_{i\in[n]},(\Reward^i)_{i\in[n]},p,\gamma)$, where $\State$ and $\Action:=\prod_i \Action^i$ are \textit{finite} state and action sets, $p(\cdot\mid s,a)$ is the state transition kernel, and $\gamma\in [0,1)$ is the discount factor. For each agent $i\in [n]$, $\Reward^i(\cdot\mid s,a)$ is a reward kernel on $\Real$, induced by a reward function $r^i:\State\times \Action\times Z^i\rightarrow\Real$ and an exogenous i.i.d. noise process $(\zeta_k^i)_{k\geq 0}$ with \textit{finite}  support $Z^i$, independent across agents, over time, and independent of the state process. Given that agents take joint action $a_k\in\Action$ and the state is $s_k\in\State$ at stage $k$, agent $i$ receives reward $r_k^i = r^i(s_k,a_k,\zeta_k^i)$. This formulation allows exploration-perturbed action execution to be equivalently modeled via stochastic rewards while treating actions as pure decisions.

Let each agent $i$ randomize their actions contingent on the current state $s$ via a Markov stationary strategy $\pi^i:\State\rightarrow \Delta(\Action^i)$, \textit{independent} of others. The agent $i$'s utility under strategy profile $\pi:= (\pi^i)_{i\in [n]}$ is the discounted sum of the rewards collected over infinite horizon and given by
\be\label{eq:utility}
U^i(\pi) = \E\left[\sum_{k=0}^{\infty} \gamma^k \reward^i(s_k,a_k)\right],
\ee
where $a_k\sim \pi(s_k)$, $s_{k+1} \sim p(\cdot\mid s_k,a_k)$, the expectation is taken over the state transitions and strategy-induced randomness, and $\reward^i(s,a) := \E_\zeta[r^i(s,a,\zeta)]$ denotes the expected reward for the pair $(s,a)$ with the expectation taken over the reward noise. 

\begin{definition}[$\varepsilon$-Nash Equilibrium]\label{def:NE}
A strategy profile $\pi_*$ is an $\varepsilon$-Nash equilibrium (also known as $\varepsilon$-approximate Markov stationary equilibrium) of stochastic game $\Markov$ with $\varepsilon \geq 0$ provided that 
\be\label{eq:NE}
U^i(\pi_*^i,\pi_*^{-i}) \geq U^i(\pi^i,\pi_*^{-i}) - \varepsilon \quad\forall \pi^i \mbox{ and }i,
\ee
where $\pi_*^{-i}:=(\pi_*^j)_{j\neq i}$. Such equilibria exist for any $\Markov$ and $\varepsilon\geq 0$ \citep{ref:Fink64}.
\end{definition}

A common justification for the game-theoretical equilibrium analysis or computation is the \textit{belief} that learning agents would reach equilibrium through non-equilibrium adaptations based on the feedback received across dynamic interactions \citep{ref:Fudenberg09}. In the following, we present independent and payoff-based learning dynamics to strengthen the predictive power of equilibrium for stochastic games. 

\section{Actor-Dual-Critic Learning Dynamics}\label{sec:ac}

In complex multi-agent environments, agents must make decisions under uncertainty and limited observability, often needing to balance short-term reactions with long-term strategic reasoning. In such cases, human decision-making uses the interplay between a \textit{fast} system that reacts quickly based on observations, and a \textit{slow} system that engages in deeper reasoning and planning \citep{ref:Kahneman11}. Motivated by this, we design a novel actor-critic architecture in which each agent is guided by two critics: a \textit{fast critic} that responds fast to immediate observations, and a \textit{slow critic} that slowly evaluates the future impact. This separation enables agents to adapt rapidly to changing local observations while simultaneously optimizing their long-term goals through slow, structured learning.


Let $v_{\pi}^i: \State\rightarrow\mathbb{R}$ and $q_{\pi}^i :\State\times \Action^i \rightarrow \mathbb{R}$, resp., denote the value function and local q-function for strategy profile $\pi$. Then, the agent utility \eqref{eq:utility} yields that
\begin{subequations}\label{eq:qv}
\begin{flalign}
&q_{\pi}^i(s,a^i) = \E_{\substack{a^{-i}\sim \pi^{-i}(s)\\ s' \sim p(\cdot\mid s,a^i,a^{-i})}} \left[\reward^i(s,a^i,a^{-i})+\gamma \cdot v_{\pi}^i(s')\right]\label{eq:q}\\
&v_{\pi}^i(s) = \E_{a^i\sim \pi^i(s)}[q_{\pi}^i(s,a^i)].\label{eq:v}
\end{flalign}
\end{subequations}
Given how the others play, i.e., $\pi^{-i}$, the local q-function depends on the expected immediate reward and the expected value of the next state while the value function depends on the agent's play, i.e., $\pi^i$, and the local q-function. We propose a new learning algorithm in which a fast critic, an actor, and a slow critic, resp., learn these three parameters $q^i,\pi^i$, and $v^i$. 

The fast critic responds fast to learn $q_\pi^i$ based on the local observations and the slow critic's value function estimate by solving \eqref{eq:q} via stochastic approximation iterations and one-stage lookahead. At stage $k$, the local q-update for the previous pair $(s,a^i) = (s_{k-1}^{},a_{k-1}^i)$ at stage $k-1$ is given by
\be\label{eq:qupdate}
q_k^i(s,a^i) = q_{k-1}^i(s,a^i)+\lambda_{k-1} \cdot \frac{r_{k-1}^i+\gamma \cdot v_{k-1}^i(s_k) - q_{k-1}^i(s,a^i)}{\pi_{k-1}^i(a^i\mid s)}
\ee
with some step size $\lambda_{k-1}\in (0,1)$ decaying slowly, where $r_{k-1}^i$ is the local immediate reward received at stage $k-1$, $s_k$ is the (next) state at stage $k$, and $\pi_{k-1}^i(a^i\mid s)\in (0,1)$ denotes the probability that $a^i$ gets played at state $s$ according to the strategy $\pi_{k-1}^i$, with a slight abuse of notation. The agent does not update the q-value for the other state-action pairs, i.e., $q_k^i(s,a^i) = q_{k-1}^i(s,a^i)$ for $(s,a^i) \neq (s_{k-1},a_{k-1}^i)$, since the agent does not observe the reward and the next state associated with those pairs. This implies that the q-values for different state-action pairs get updated asynchronously. The normalization of the second term on the right-hand side of \eqref{eq:qupdate} by the action-execution probability $\pi_{k-1}^i(a^i\mid s)\in (0,1)$ ensures that they all get updated at the same rate in expectation, as in the individual Q-learning dynamics \citep{ref:Leslie05}.

For each agent $i$, define the \textit{exploration} kernel $\explore^i(\cdot\mid a^i) \in \Delta(\Action^i)$ by
\be
\explore^i(\tilde{a}^i\mid a^i) := \left\{
\begin{array}{ll}
1-\epsilon + \frac{\epsilon}{|\Action^i|}&\mbox{if } \tilde{a}^i = a^i\\
\frac{\epsilon}{|\Action^i|}&\mbox{if } \tilde{a}^i \neq a^i
\end{array}
\right. 
\ee
for some exploration rate $\epsilon\in(0,1)$.
Then, the actor (determining how the agent plays) can iteratively improve its strategy $\pi^i_k$ based on the $\epsilon$-best response to the fast critic at the current state according to
\be\label{eq:actor}
\pi_{k}^i(s) = \pi_{k-1}^i(s) + \alpha_{k-1} \cdot  (\BR(q_{k-1}^i(s,\cdot)) - \pi_{k-1}^i(s))
\ee
for all $s$ with some step size $\alpha_{k-1}\in(0,1)$ decaying at a moderate rate, where
\begin{subequations}\label{eq:best}
\begin{flalign}
&\BR(q^i_{k-1}(s,\cdot)) := \explore^i(\cdot\mid \BRfree(q^i_{k-1}(s,\cdot)))\\
&\BRfree(q^i_{k-1}(s,\cdot))\in \argmax_{\tilde{a}^i\in \Action^i} \{q^i_{k-1}(s,\tilde{a}^i)\}.
\end{flalign}
\end{subequations} 
Notably, by \eqref{eq:actor} and \eqref{eq:best}, we can decompose the actor as 
\be\label{eq:decompose}
\pi_k^i(s) = (1-\epsilon)\mu_k^i(s) + \epsilon \frac{\mathbf{1}}{|\Action^i|},
\ee
where $\mu_k^i(s)$ evolves according to
\be\label{eq:actormu}
\mu_{k}^i(s) = \mu_{k-1}^i(s) + \alpha_{k-1} \cdot  (\BRfree(q^i_{k-1}(s,\cdot))- \mu_{k-1}^i(s))\quad\forall s.
\ee

The slow critic evaluates the actor strategy based on the fast critic and the actor iteratively. At stage $k$, the local value function update is given by
\be\label{eq:vupdate}
v_k^i(s) = v_{k-1}^i(s) + \beta_{k-1} \cdot (\pi_{k-1}^i(s)^Tq_{k-1}^i(s,\cdot) - v_{k-1}^i(s))
\ee
for all $s$, with some step size $\beta_{k-1}\in (0,1)$ decaying fast. Here, we view $\pi_{k-1}^i(s)$ and $q_{k-1}^i(s,\cdot)\in \mathbb{R}^{|\Action^i|}$ as $|\Action^i|$-dimensional vectors. \Cref{alg:ADC} tabulates the complete description of the \textit{Actor-Dual-Critic} dynamics.  

\begin{algorithm}[t!]
    \caption{Actor-Dual-Critic Dynamics}
    \begin{algorithmic}\label{alg:ADC}
    \small
    \STATE {\bfseries initialization:} $q_0^i(s,a^i) = v_0^i(s) = 0$ for all $(s,a^i)$ and $\pi_0^i(s) = \mathrm{Uniform}(\Action^i)$ for all $s$ 
    \FOR{each stage $k=0,1,\ldots$}  
    \STATE observe $s_k$
    \IF {$k>0$}
    \STATE for all $(s,a^i)$, update  
    \begin{flalign*}
    &q_k^i(s,a^i) = q_{k-1}^i(s,a^i)+\lambda_{k-1}\cdot \ind{(s,a^i)=(s_{k-1}^{},a_{k-1}^i)} \frac{r_{k-1}^i+\gamma \cdot v_{k-1}^i(s_k) - q_{k-1}^i(s,a^i)}{\pi_{k-1}^i(a^i\mid s)}\\
    &\pi_{k}^i(s) = \pi_{k-1}^i(s) + \alpha_{k-1} \cdot  (\BR(q_{k-1}^i(s,\cdot)) - \pi_{k-1}^i(s)) \\
    &v_k^i(s) = v_{k-1}^i(s) + \beta_{k-1} \cdot  (\pi_{k-1}^i(s)^Tq_{k-1}^i(s,\cdot) - v_{k-1}^i(s))
    \end{flalign*}
    \ENDIF
    \STATE play $a_k^i\sim \pi_k^i(s_k)$
    \STATE receive $r_k^i = r^i(s_k,a_k,\zeta_k^i)$
    \ENDFOR
    \end{algorithmic}
\end{algorithm}

The actor-dual-critic dynamics are aligned with the existing independent and payoff-based methods proposed for stochastic games, such as those in \citep{ref:Sayin21} and \citep{ref:Chen23}. \cite{ref:Sayin21} extends the individual-Q learning framework of \cite{ref:Leslie05} by integrating value iteration in a two-timescale setup, where values evolve more slowly. However, their agents do not follow an actor; instead, they directly respond to the local q-function through a smoothed best response. In contrast, \cite{ref:Chen23} builds on the actor-critic dynamics of \cite{ref:Leslie06}, adapting them to stochastic games using a nested loop where value updates occur only at the end of each inner loop. 

Our actor-dual-critic dynamics differ from \cite{ref:Chen23} in three key ways: (i) value functions are updated continuously at a slower timescale rather than in episodic batches; (ii) actor strategies are improved using $\epsilon$-greedy responses instead of smoothed ones; and (iii) local q-function estimates use normalized step sizes for each action, yielding a more tractable and robust formulation. Furthermore, both \cite{ref:Sayin21} and \cite{ref:Chen23} provide convergence guarantees for their algorithms in two-agent zero-sum stochastic games using asymptotic and non-asymptotic analyses, respectively. In the following section, we establish the convergence of the actor-dual-critic dynamics in \textit{both} two-agent zero-sum and multi-agent identical-interest stochastic games.

\section{Main Results}\label{sec:results}
To present the convergence results, we first address the exploration-perturbed execution and coupled dynamics across the actor and critics. 

\subsection{Effective Stochastic Game Formulation}
The actor-update in \Cref{alg:ADC} relies on $\epsilon$-best responses, which induce persistent and independent exploration in action execution. Rather than modeling this exploration explicitly through mixed strategies, we equivalently incorporate it into the reward and transition mechanisms so that we can view $\epsilon$-best responses as pure best responses in an \emph{effective stochastic game}, characterized by the tuple $\Markoveps=(\State,(\Action^i)_{i\in[n]},(\Rewardeps^i)_{i\in[n]},\peps,\gamma)$. Due to independent exploration, define the joint exploration kernel by $\explore(\tilde{a}\mid a) := \prod_{j\in [n]}\explore^j(\tilde{a}^j\mid a^j)$. Then, the \textit{effective} reward and transition kernels are given by
\begin{subequations}
\begin{flalign}
&\Rewardeps^i(\cdot\mid s,a) := \sum_{\tilde{a}\in \Action} \Reward^i(\cdot\mid s,\tilde{a}) \explore(\tilde{a}\mid a),\\
&\peps(\cdot \mid s,a) :=\sum_{\tilde{a}\in \Action} p(\cdot\mid s,\tilde{a})\explore(\tilde{a}\mid a).
\end{flalign}
\end{subequations}
Correspondingly, the effective \textit{expected} rewards and transitions are given by
\be\label{eq:effectiverewards}
\rewardeps^i (s,a) = \E_{\tilde{a}\sim \explore(\cdot\mid a)}[\reward^i(s,\tilde{a})]\quad\mbox{and}\quad\opeps^i (s'\mid s,a) = \E_{\tilde{a}\sim \explore(\cdot\mid a)}[\peps(s' \mid s,\tilde{a})].
\ee
Let $\mu:=(\mu^i:\State\rightarrow \Delta(\Action^i))_{i\in [n]}$ be the \textit{exploration-free} joint strategy associated with the exploration-perturbed execution strategy $\pi$ in the original game, i.e., $\pi^i(s) = (1-\epsilon)\mu^i(s) + \epsilon \mathrm{Uniform}(\Action^i)$.
When agents play according to $\pi$ in the original stochastic game, the effective utility of agent $i$ in the effective stochastic game is given by
\be\label{eq:effectiveutility}
\Ueps^i(\mu) = \E\left[\sum_{k=0}^\infty\rewardeps^i(s_k,a_k)\right],
\ee
where $a_k\sim \mu(s_k)$ and $s_{k+1} \sim \opeps(\cdot\mid s_k,a_k)$. Therefore, exploration-free strategies in this effective stochastic game $\Markoveps$ are distributionally equivalent to exploration-perturbed strategies in the original game $\Markov$.

\subsection{Three-timescale Dynamics}
In the underlying stochastic game, the interplay between the fast critic, the actor, and the slow critic
yields a cyclic interconnection between them. The timescale separation of their evolution can break this cyclic nature by weakening the coupling between them. Particularly, \Cref{alg:ADC} uses three step sizes $\lambda_k$, $\alpha_k$ and $\beta_k$, decaying at different rates to ensure such timescale separation.  

\begin{assumption}{Step Sizes}\label{assume:step}
The step sizes are given by 
$\lambda_k = (k+1)^{-\rho_{\lambda}}$, $\alpha_k = (k+1)^{-\rho_{\alpha}}$, $\beta_k = (k+1)^{-1}$ for some $\frac{1}{2}<\rho_{\lambda}< \rho_{\alpha}<1$.
\end{assumption}

\Cref{assume:step} ensures that the step sizes satisfy the standard stochastic approximation conditions, i.e., square-summable yet not summable, while $\alpha_k \in o(\lambda_k)$ and $\beta_k \in o(\alpha_k)$, implying three-timescale dynamics. In particular, the step size $\lambda_k$ decays to zero at the slowest rate yielding that the strategies and value estimates are \textit{quasi-stationary} in the local q-update such that $q_k^i(s,\cdot)$ can track $Q_k^i(s,\cdot)\pi_k^{-i}$, where
\be\label{eq:Q}
Q_k^i(s,a) := \reward^i(s,a) + \gamma \sum_{s'\in \State}p(s'\mid s,a) v_k^i(s')\quad\forall (s,a)
\ee  
is the underlying global Q-function associated with the value function estimate $v_k^i$. In the actor update, the step size $\alpha_k$ decays to zero at a moderate rate yielding that the value estimates are \textit{quasi-stationary} while local q-function estimates $q_k^i(s,\cdot)$ are tracking $Q_k^i(s,\cdot)\pi_k^{-i}$. Therefore, in each stage $k$, actors are approximately playing in an induced stage game whose payoffs are $Q_k^i(s,\cdot)$'s associated with the current state. The quasi-stationarity of value function estimates implies the quasi-stationarity of (implicit) global Q-function estimates by \eqref{eq:Q}. Correspondingly, actors are $\epsilon$-greedy responding to each other in the repeated play of quasi-stationary stage games and can track equilibrium of the stage game if it has certain structures. In the slow-critic update, the step size $\beta_k$ decays to zero at the fastest rate yielding that the target term $\pi_k^i(s)^Tq_k^i(s,\cdot)$ can track equilibrium associated with the current value function estimates.

\subsection{Convergence Results}
We define the \textit{Nash gap} and \textit{effective Nash gap} for exploration-perturbed $\pi$ and exploration-free $\mu$ strategy profiles, resp., by
\begin{subequations}\label{eq:gap}
\begin{flalign}
&\Gap(\pi) := \max_{i\in [n]} \Big(\max_{\hat{\pi}^i} U^i(\hat{\pi}^i,\pi^{-i}) - U^i(\pi^i,\pi^{-i})\Big),\\
&\Gap_\epsilon(\mu) := \max_{i\in [n]} \Big(\max_{\hat{\mu}^i} \Ueps^i(\hat{\mu}^i,\mu^{-i}) - \Ueps^i(\mu^i,\mu^{-i}).\Big).\label{eq:gapeffective}
\end{flalign}
\end{subequations}
Given Markov stationary strategies $\pi^{-i}$ or $\mu^{-i}$, the problems $\max_{\hat{\pi}^i} U^i(\hat{\pi}^i,\pi^{-i})$ or $\max_{\hat{\mu}^i} \Ueps^i(\hat{\mu}^i,\mu^{-i})$ induce Markov decision processes, for which Markov stationary solutions always exist \citep{ref:Sutton18}. Furthermore, these gaps are non-negative by definition.

\begin{proposition*}\label{prop:NG}
Consider an exploration-free Markov stationary strategy profile $\mu$ and the associated exploration-perturbed strategy profile $\pi$, i.e., $\pi^i(s) = (1-\epsilon)\mu^i(s) + \epsilon u^i(s)$ for all $i$ and $s$. If $\Gap_\epsilon(\mu)=0$, then $\Gap(\pi)\leq \varepsilon$, where
\be\label{eq:threshold}
\varepsilon:= \frac{2\epsilon}{(1-\gamma)^2}\cdot \max_{(i,s,a)}  |\reward^i(s,a)|
\ee 
and, therefore, $\pi$ is $\varepsilon$-Nash equilibrium.
\end{proposition*}

\begin{proof}
Fix an agent $i$ and consider an arbitrary Markov stationary strategy profile $\tilde{\mu}^i$ in the effective game and the associated exploration-perturbed strategy profile $\tilde{\pi}$ in the original game, i.e., we have $\tilde{\pi}^i(s) = (1-\epsilon)\tilde{\mu}^i(s) + \epsilon u^i$, where $u^i := \frac{\mathbf{1}}{|\Action^i|}$ denotes the uniform distribution over $\Action^i$ for notational simplicity. By the construction of the effective expected reward and transitions \eqref{eq:effectiverewards} in the effective utility \eqref{eq:effectiveutility} and $\Gap_\epsilon(\mu)=0$ as assumed in the proposition statement, we have
\be\label{eq:UUU}
U^i(\pi) = \Ueps^i(\mu) \geq \Ueps^i(\tilde{\mu}^i,\mu^{-i}) = U^i(\tilde{\pi}^i,\pi^{-i}).
\ee

Given a Markov stationary strategy profile $\bar{\pi}$, a standard solution for the value of each state in the original stochastic game is given by 
\be
\vvec_{\bar{\pi}}^i = (1-\gamma \Pvec_{\bar{\pi}})^{-1} \rvec_{\bar{\pi}}^i,
\ee
where $\vvec_{\bar{\pi}}^i \in \mathbb{R}^{|\State|}$ is the value vector, $\Pvec_{\bar{\pi}}\in \mathbb{R}^{|\State|\times |\State|}$ is the (stochastic) transition matrix, and $\rvec_{\bar{\pi}}^i \in \mathbb{R}^{|\State|}$ is the reward vector. For example, the entry $\Pvec_{\bar{\pi}}(s,s') = \E_{a\sim \bar{\pi}(s)}[p(s'\mid s,a)]$ and $\rvec_{\bar{\pi}}^i(s)=\E_{a\sim \bar{\pi}(s)}[\reward^i(s,a)]$. Correspondingly, we have the identity
\be\nn
\vvec_{\tilde{\pi}^i,\pi^{-i}} - \vvec_{\tilde{\mu}^i,\pi^{-i}} = (1-\gamma \Pvec_{\tilde{\pi}^i,\pi^{-i}})^{-1}(\rvec_{\tilde{\pi}^i,\pi^{-i}}^i - \rvec_{\tilde{\mu}^i,\pi^{-i}}^i + \gamma (\Pvec_{\tilde{\pi}^i,\pi^{-i}} - \Pvec_{\tilde{\mu}^i,\pi^{-i}})\vvec_{\tilde{\mu}^i,\pi^{-i}}^i).
\ee 
Taking the infinity norm of both hand side and applying triangle inequality, we obtain
\begin{flalign}
\|\vvec_{\tilde{\pi}^i,\pi^{-i}} - \vvec_{\tilde{\mu}^i,\pi^{-i}}\|_{\infty} \leq &\;\|(1-\gamma \Pvec_{\tilde{\pi}^i,\pi^{-i}})^{-1}\|\|\rvec_{\tilde{\pi}^i,\pi^{-i}}^i - \rvec_{\tilde{\mu}^i,\pi^{-i}}^i\|_\infty\nn\\
&+\gamma \|(1-\gamma \Pvec_{\tilde{\pi}^i,\pi^{-i}})^{-1}\|\|(\Pvec_{\tilde{\pi}^i,\pi^{-i}} - \Pvec_{\tilde{\mu}^i,\pi^{-i}})\vvec_{\tilde{\mu}^i,\pi^{-i}}^i\|_\infty.\label{eq:bb}
\end{flalign}
Since $\tilde{\pi}^i(s) = (1-\epsilon)\tilde{\mu}^i(s) + \epsilon u^i$, we have $\rvec_{\tilde{\pi}^i,\pi^{-i}}^i = (1-\epsilon) \rvec_{\tilde{\mu}^i,\pi^{-i}}^i + \epsilon \rvec_{u^i,\pi^{-i}}^i$ and  $\Pvec_{\tilde{\pi}^i,\pi^{-i}} = (1-\epsilon) \Pvec_{\tilde{\mu}^i,\pi^{-i}} + \epsilon \Pvec_{u^i,\pi^{-i}}$. This yields that
\be
\|\rvec_{\tilde{\pi}^i,\pi^{-i}}^i - \rvec_{\tilde{\mu}^i,\pi^{-i}}^i\|_{\infty} = \epsilon \|\rvec_{\tilde{\mu}^i,\pi^{-i}}^i - \rvec_{u^i,\pi^{-i}}^i\|_{\infty}\leq 2 \epsilon \max_{(j,s,a)}  |\reward^j(s,a)|\label{eq:b1}
\ee
and
\begin{flalign}
\|(\Pvec_{\tilde{\pi}^i,\pi^{-i}} - \Pvec_{\tilde{\mu}^i,\pi^{-i}})\vvec_{\tilde{\mu}^i,\pi^{-i}}^i\|_\infty &= \epsilon \|(\Pvec_{\tilde{\mu}^i,\pi^{-i}} - \Pvec_{u^i,\pi^{-i}})\vvec_{\tilde{\mu}^i,\pi^{-i}}^i\|_\infty \\
&\leq 2\epsilon \|\vvec_{\tilde{\mu}^i,\pi^{-i}}^i\|_\infty \leq \frac{2\epsilon}{1-\gamma} \max_{(j,s,a)}  |\reward^j(s,a)|\label{eq:b2}
\end{flalign}
as the induced norm $\|\Pvec_\pi\|\leq 1$ for all $\pi$ since the associated transition matrices are stochastic matrices for any strategy profile. Furthermore, the Neumann series gives
\be
\|(1-\gamma \Pvec_{\tilde{\pi}^i,\pi^{-i}})^{-1}\| = \left\|\sum_{t=0}^\infty \gamma^t \Pvec_{\tilde{\pi}^i,\pi^{-i}}^t\right\|\leq \frac{1}{1-\gamma}.\label{eq:b3}
\ee
Plugging the bounds \eqref{eq:b1}, \eqref{eq:b2}, and \eqref{eq:b3} into \eqref{eq:bb}, we obtain
\be\label{eq:vbound}
\|\vvec_{\tilde{\pi}^i,\pi^{-i}} - \vvec_{\tilde{\mu}^i,\pi^{-i}}\|_{\infty}\leq \frac{2\epsilon}{(1-\gamma)^2}\max_{(j,s,a)}  |\reward^j(s,a)|=\varepsilon.
\ee
By the definition of value and utility functions, \eqref{eq:vbound} yields that $U^i(\tilde{\pi}^i,\pi^{-i}) \geq U^i(\tilde{\mu}^i,\pi^{-i}) - \varepsilon$ for all $\tilde{\mu}^i$. Thus, by \eqref{eq:UUU}, we can conclude that $U^i(\pi) \geq U^i(\tilde{\mu}^i,\pi^{-i}) - \varepsilon$ for all $\tilde{\mu}^i$, which completes the proof.
\end{proof}

Recall the asynchronous update of fast-critic parameters due to their dependence on observations. Agents do not know the counterfactuals about the actions not taken or states not visited. Therefore, whether every state gets visited sufficiently many times plays an important role in convergence guarantees. To this end, we make the following assumption on the underlying state transition dynamics. 

\begin{assumption}{Reachability}\label{assume:reach}
We have $p(s'\mid s, a) >0$ for any $(s,a,s')$.
\end{assumption}

Such a reachability assumption is used in \citep{ref:Leslie20,ref:Sayin22}. Although we can assume weaker reachability conditions as in \citep{ref:Sayin21,ref:Chen23} for zero-sum stochastic games, that would lead to  more involved technical challenges for identical-interest stochastic games and, therefore, left as a future research direction. 

Now, we can first present the convergence guarantees of the actor-dual-critic dynamics for two-agent zero-sum stochastic games $\Markov_z = (\State,(\Action^i)_{i\in[n]},(\Reward^i)_{i\in[n]},p,\gamma)$, where the induced expected reward functions are $\reward^1(s,a) + \reward^2(s,a) = 0$ for all $(s,a)$. The following theorem shows the decaying Nash-gap for the zero-sum case. 

\begin{theorem}{Zero-sum Stochastic Games}\label{thm:mainZS}
Consider a two-agent zero-sum stochastic game $\Markov_z$.  Suppose that all agents independently follow the actor-dual-critic dynamics, as described in \Cref{alg:ADC}, and \Cref{assume:step,assume:reach} hold. Then, the exploration-free actor strategy $\mu_k$, evolving as in \eqref{eq:actormu}, satisfy $\Gap_\epsilon(\mu_k) \rightarrow 0$ almost surely and, therefore, we almost surely have
\be\label{eq:NGZ}
\limsup_{k\rightarrow\infty} \Gap(\pi_k) \leq \varepsilon.
\ee
\end{theorem}

Intuitively, the actor-dual-critic dynamics are expected to reach equilibrium for zero-sum games due to their similarity with the dynamics in \citep{ref:Sayin21,ref:Chen23}. More explicitly, the three timescale dynamics described above implies that value function estimates track the minimax equilibrium of the stage games associated with the current value estimates. Then, the contraction property of the \cite{ref:Shapley53}'s minimax operator yields the convergence with some approximation error induced by the inherent exploration. Therefore, its convergence proof has a similar flavor with the analyses in \citep{ref:Sayin21,ref:Chen23}. 

We next present the convergence guarantees for multi-agent $(n\geq 2)$ identical-interest stochastic games  $\Markov_p = (\State,(\Action^i)_{i\in[n]},(\Reward^i)_{i\in[n]},p,\gamma)$, where the induced expected reward functions are $\reward^i(s,a) = \reward(s,a)$ for all $(i,s,a)$ for some common reward function $\reward:\State\times \Action \rightarrow \mathbb{R}$. We make the following additional assumption on step sizes:

\begin{assumption*}\label{assume:additional}
The step sizes satisfy $\rho_\alpha \geq \frac{3}{2}\rho_\lambda$ and $\rho_\alpha+\rho_\lambda > \frac{3}{2}$.
\end{assumption*}

The following theorem shows the decaying Nash-gap for the identical-interest case. 

\begin{theorem}{Identical-interest Stochastic Games}\label{thm:mainII}
Consider a multi-agent $(n\geq 2)$ identical-interest stochastic game $\Markov_p$.  Suppose that all agents independently follow the actor-dual-critic dynamics, as described in \Cref{alg:ADC}, and \Cref{assume:step,assume:reach,assume:additional} hold. Then, the exploration-free actor strategy $\mu_k$, evolving as in \eqref{eq:actormu}, satisfy $\Gap_\epsilon(\mu_k) \rightarrow 0$ almost surely and, therefore, we almost surely have
\be\label{eq:NGP}
\limsup_{k\rightarrow\infty} \Gap(\pi_k) \leq \varepsilon.
\ee
\end{theorem}

We highlight that our main technical contribution is on the convergence analysis for identical-interest stochastic games, which has not been addressed in \citep{ref:Sayin21,ref:Chen23}. The convergence analysis for zero-sum games is not applicable for identical-interest games due to  multiple equilibrium values and the lack of contraction property. \cite{ref:Monderer96b} addressed the convergence of fictitious play in identical-interest \textit{normal-form} games played repeatedly by showing the quasi-monotonicity of the game value while both agents are seeking for higher payoffs. It is not exactly a monotonic increase due to the lack of access to opponent strategy, simultaneous update of actions in addition to their randomness. \cite{ref:Baudin22} has extended this approach to identical-interest stochastic games through continuous-time two-timescale stochastic differential inclusion methods to show the convergence of fictitious play in such games. However, here the lack of access to opponent actions and maintaining local value function estimates pose a significant technical challenge. \cite{ref:Sayin22b} has addressed the independent value function updates for stochastic games with single controllers through a discrete-time analysis. They have also focused on fictitious play dynamics where agents have access to opponent actions. 

\Cref{alg:ADC} differs from fictitious-play-type dynamics by maintaining an actor and updating the game values based on the actor strategy (rather than the best response). Though agents maintain independent value updates, the value estimates across agents get close to each other as the local q-function estimates $q_k^i(s,\cdot)$ track $Q_k^i(s,\cdot)\pi_k^{-i}$. By characterizing this tracking rate, we can quantify their impact on different value function estimates. Furthermore, the randomization of the actions according to actor pose a challenge for the discrete-time analysis. 

\section{Convergence Analysis}\label{sec:proof}

\Cref{alg:ADC} yields that the iterates $q_k^i,\pi_k^i$ and $v_k^i$ for stage $k$ evolve according to the following update rules:
\begin{subequations}\label{eq:update}
\begin{flalign}
&q_{k+1}^i(s,a^i) = q_k^i(s,a^i) + \lambda_k\cdot\frac{\ind{(s,a^i)=(s_k^{},a_k^i)}}{\pi_k^i(a^i\mid s)}\cdot (\hatq_k^i - q_k^i(s,a^i))\label{eq:localQupdate}\\
&\pi_{k+1}^i(s) = \pi_k^i(s)+\alpha_k\cdot (\hatpi_k^i(s) - \pi_k^i(s))\label{eq:piupdate}\\
&v_{k+1}^i(s) = v_k^i(s) + \beta_k \cdot (\hat{v}_k^i(s) - v_k^i(s)),\label{eq:vupdate}
\end{flalign}
\end{subequations}
for all $(s,a^i)$, where local reward $r_k^i = r^i(s_k,a_k,\zeta_k^i)$, local action $a_k^i\sim \pi_k^i(s_k)$, and
\be
 \hatq_k^i = r_k^i + \gamma \cdot v_k^i(s_{k+1}), \quad \hatpi_k^i(s):=\BR(q_k^i(s,\cdot)), \quad\hat{v}_k^i(s) := \pi_k^i (s)^Tq_k^i(s,\cdot).
\ee
The normalization by $\pi_k^i(a^i\mid s)$ in \eqref{eq:localQupdate} can make the normalized step size larger than $1$. However, the policy update taking step toward the $\epsilon$-greedy response $\hatpi_k^i(s)$ ensures that there is a uniform lower bound, e.g., $\epsilon/|\Action^i|$, on the probability of the actions taken. Therefore, decaying $\lambda_k^i$ (by \Cref{assume:step}) yields that normalized step sizes become less than $1$ in \textit{finite} time and all iterates remain \textit{bounded}.

By \eqref{eq:vupdate} and \eqref{eq:Q}, the (implicit) global Q-function evolves according to 
\be
Q_{k+1}^i(s,a) = Q_k^i(s,a) + \beta_k \cdot (\hatQ_k^i(s,a) - Q_k^i(s,a)), \label{eq:globalQupdate}
\ee
where we define 
\be\label{eq:hatQ}
\hatQ_k^i(s,a) := \reward^i(s,a) + \gamma\sum_{s'\in\State} p(s'\mid s,a) \cdot \hatv_k^i(s').
\ee

We address the convergence of \eqref{eq:update} and \eqref{eq:globalQupdate}, i.e., \Cref{alg:ADC}, for zero-sum and identical-interest stochastic games, resp., in \Cref{sec:zerosum} and \Cref{sec:potential}. 

\subsection{Proof for Two-agent Zero-sum Stochastic Games}\label{sec:zerosum}
The proof for the zero-sum games is relatively straight-forward based on \citep{ref:Sayin22,ref:Sayin21}. Particularly, the three timescale dynamics ensure that local q-value estimates can track the one associated with the (implicit) global Q-value estimate and the opponent policy, i.e., $\|q_k^i(s)-Q_k^i(s,\cdot)\pi_k^{-i}(s)\|\rightarrow 0$ almost surely. Therefore, agents are essentially playing exploration-perturbed best response dynamics at the timescale of the actor updates. Correspondingly, the exploration-free actor, as described in \eqref{eq:actormu}, can be written in the generalized-fictitious-play structure \citep{ref:Leslie06} as
\be\label{eq:muZS}
\mu_{k+1}^i(s) - (1-\alpha_k)\cdot \mu_k^i(s) \in \alpha_k\cdot \mathrm{BR}_k(\mu_k^{-i}(s)), 
\ee
where 
\begin{flalign}
\mathrm{BR}_k(\mu_k^{-i}(s)) := \{\mu^i\in\Delta(\Action^i): &\;(\mu^i)^T\overline{Q}_k^i(s,\cdot)\mu_k^{-i}(s) \geq (\tilde{\mu}^i)^T\overline{Q}_k^i(s,\cdot)\mu_k^{-i}(s) - z_k^i(s),\nn\\
&\forall \tilde{\mu}^i\in\Delta(\Action^i)\}
\end{flalign}
for some $z_k^i(s)$ decaying to zero almost surely and the \textit{effective} global Q-function estimate is defined by
\be\label{eq:QBZ}
\overline{Q}_k^i(s,\cdot) := (1-\epsilon) Q_k^i(s,\cdot) + \frac{\epsilon}{|\Action^{-i}|}Q_k^i(s,\cdot)\mathbf{1}\mathbf{1}^T + \frac{\epsilon}{|\Action^{i}|}\mathbf{1}\mathbf{1}^TQ_k^{i}(s,\cdot).
\ee

However, $Q_k^i(s,\cdot)\neq -Q_k^{-i}(s,\cdot)^T$ and, therefore, $\overline{Q}_k^i(s,\cdot)\neq -\overline{Q}_k^{-i}(s,\cdot)^T$ in general due to the independent value updates \eqref{eq:vupdate}. The differences in the corresponding global Q-function estimates pose a challenge. The previous works such as \cite{ref:Sayin22,ref:Sayin21} have already addressed this issue by crafting a new Lyapunov function for near zero-sum games and shown that the stage games become zero-sum asymptotically, i.e., $\|\overline{Q}_k^i(s,\cdot)+\overline{Q}_k^{-i}(s,\cdot)^T\|\rightarrow 0$ almost surely under \Cref{assume:step,assume:reach}. 

Since the stage games become zero-sum games asymptotically and \eqref{eq:muZS} is a generalized fictitious play dynamics, exploration-free actor $\mu^i(s)$ can track the \textit{minimax} equilibrium in the associated effective stage games with the payoffs $\overline{Q}_k^i(s,\cdot)$ and $-\overline{Q}_k^{i}(s,\cdot)^T$ for each $s$. In other words, we have 
\be\label{eq:minmax}
\left|\mu_k^i(s)^T\overline{Q}_k^i(s,\cdot)\mu_k^{-i}(s)-\max_{\mu^i\in\Delta(\Action^i)}\min_{\mu^{-i}\in\Delta(\Action^{-i})} (\mu^i)^T\overline{Q}_k^i(s,\cdot)\mu^{-i}\right|\rightarrow 0
\ee
almost surely. On the other hand, by \eqref{eq:decompose} and \eqref{eq:QBZ}, we can show that 
\begin{flalign}
\pi_k^i(s)^TQ_k^i(s,\cdot)\pi_k^{-i}(s)
&= (1-\epsilon) \cdot \mu_k^i(s)^T\overline{Q}_k^i(s,\cdot)\mu_k^{-i}(s).\label{eq:newidentity}
\end{flalign}
By \eqref{eq:minmax}, \eqref{eq:newidentity}, and $\|q_k^i(s)-Q_k^i(s,\cdot)\pi_k^{-i}(s)\|\rightarrow 0$, we can write \eqref{eq:vupdate}  as
\be
v_{k+1}^i(s) = v_k^i(s) + \beta_k \cdot \Big((1-\epsilon) \max_{\mu^i}\min_{\mu^{-i}} (\mu^i)^T\overline{Q}_k^i(s,\cdot)\mu^{-i} + \bar{z}_k^i(s) - v_k^i(s)\Big),
\ee
where $\bar{z}_k^i(s)$ is some asymptotically negligible error term.

Next, based on \eqref{eq:QBZ}, we define the operator
\be
\mathcal{T}(Q) = (1-\epsilon)Q + \frac{\epsilon}{|\Action^{-i}|}Q\mathbf{1}\mathbf{1}^T + \frac{\epsilon}{|\Action^{i}|}\mathbf{1}\mathbf{1}^T Q.
\ee
Then, the maximum norm of the difference is bounded by
\begin{flalign}
\|\mathcal{T}(Q)-\mathcal{T}(Q')\|_{\max} &\leq (1-\epsilon)\|Q-Q'\|_{\max} + \frac{\epsilon}{|\Action^{-i}|} \|(Q-Q')\mathbf{1}\mathbf{1}^T\|_{\max} \nn\\
&\hspace{.2in} + \frac{\epsilon}{|\Action^{i}|}\|\mathbf{1}\mathbf{1}^T (Q-Q')\|_{\max}\\
&\leq (1+\epsilon)\|Q-Q'\|_{\max}
\end{flalign}
Therefore, the non-expansiveness of the maximin function yields that
\begin{flalign}
(1-\epsilon)&\left|\max_{\mu^i\in\Delta(\Action^i)}\min_{\mu^{-i}\in\Delta(\Action^{-i})} (\mu^i)^T\mathcal{T}(Q)\mu^{-i}- \max_{\mu^i\in\Delta(\Action^i)}\min_{\mu^{-i}\in\Delta(\Action^{-i})} (\mu^i)^T\mathcal{T}(Q')\mu^{-i}\right|\nn\\
&\leq (1-\epsilon) \|\mathcal{T}(Q)-\mathcal{T}(Q')\|_{\max} \leq (1-\epsilon^2) \|Q-Q'\|_{\max}\leq \|Q-Q'\|_{\max}.
\end{flalign}
Furthermore, by the definition \eqref{eq:Q}, we have $\|Q-Q'\|_{\max}\leq \gamma \|v-v'\|_\infty$. The (perturbed) $\gamma$-contraction of the value update and \Cref{assume:step} imply its almost-sure convergence and, therefore, the almost-sure decay of the effective Nash gap \eqref{eq:gapeffective} since the exploration-free actor tracks the associated minimax equilibrium. Then, by \Cref{prop:NG}, we can conclude \eqref{eq:NGZ}. This summarizes the key steps in the proof of \Cref{thm:mainZS}. 

We have not provided the entire proof in full details since the convergence is intuitive given the previous results \citep{ref:Sayin22,ref:Sayin21}, the technical contribution is relatively marginal, and the space is limited. Instead, we mainly focus on the convergence of \Cref{alg:ADC} for identical-interest games, different from the existing works \citep{ref:Sayin22,ref:Sayin21,ref:Chen23}.

\subsection{Proof for Multi-agent Identical-interest Stochastic Games}\label{sec:potential}
Recall the challenge that in identical-interest games (where $r^i(s,a)=r(s,a)$ for all $(i,s,a)$ for some common reward function $r:\State\times \Action\rightarrow\mathbb{R}$), multiple equilibria can have different game values and the value iteration based on Nash equilibria does not necessarily have the contraction property. To resolve these challenges, we leverage the \textit{quasi-monotonicity} of the value function estimates, specifically the global Q-function estimates \eqref{eq:globalQupdate}. We particularly focus on whether the accumulated impact of the innovation terms in the global-Q update \eqref{eq:globalQupdate} becomes asymptotically non-negative such that global Q-values show a quasi-monotonic increase. Then, the boundedness of the iterates yields their convergence.

More formally, we seek to show that
\be\label{eq:liminfcondition}
\liminf_{k\rightarrow \infty} \inf_{k_2\geq k_1\geq k} \{Q_{k_2}^i(s,a)-Q_{k_1}^i(s,a)\} \geq 0\quad\forall (i,s,a),
\ee
which implies that for any $\epsilon>0$, there exists $K$ such that for each $(s,a)$, any future Q-values $k'\geq k$ are at most as small as $Q_K^i(s,a)-\epsilon$, e.g., see \citep[Proposition 6]{ref:Sayin22b}. Denote the innovation in  \eqref{eq:globalQupdate} by $Y_k^i \equiv \hatQ_k^i - Q_k^i$. Then, \eqref{eq:liminfcondition} can be written as
\be\label{eq:liminfYcondition}
\liminf_{k\rightarrow \infty} \inf_{k_2\geq k_1\geq k} \sum_{l =k_1}^{k_2-1} \beta_l \cdot Y_l^i(s,a) \geq 0\quad\forall (i,s,a).
\ee
Henceforth, we consider \textit{two-agent} games, where agent $i$ is the typical player while agent $-i$ is the typical opponent, for notational simplicity. However, the results can be extended to the games with more than two agents straightforwardly. 

To show \eqref{eq:liminfYcondition}, we focus on the evolution of the innovation $Y_k^i$. By \eqref{eq:globalQupdate}, we have
\begin{flalign}
Y_{k+1}^i(s,a) = (1-\beta_k) Y_k^i(s,a) + \gamma \sum_{s'\in \State}p(s'\mid s,a) (\hatv_{k+1}^i(s') - \hatv_k^i(s')).\label{eq:Ydiff}
\end{flalign}
After some algebra, we can show that
\begin{flalign}
\hatv_{k+1}^i(s) - \hatv_k^i(s) = &\;\beta_k\pi_k^i(s)^TY_k^i(s,\cdot)\pi_k^{-i}(s) + \alpha_k(\hatpi_k^i(s) - \pi_k^i(s))^T(q_k^i(s,\cdot) +\delta_k^i(s)) \nn\\
&+ \alpha_k(\hatpi_k^{-i}(s) - \pi_k^{-i}(s))^T\big(q_k^{-i}(s,\cdot) + \delta_k^{-i}(s) + \Delta_k(s)\pi_k^i(s)\big) \nn\\
&-\pi_{k+1}^i(s)^T\delta_{k+1}^i(s) + \pi_k^i(s)^T\delta_k^i(s) + O(\alpha_k^2),\label{eq:udiff2}
\end{flalign} 
where we define the tracking error and mismatch error, resp., by
\be\label{eq:deltas}
\delta_k^i(s) := Q_k^i(s,\cdot)\pi_k^{-i}(s) - q_k^i(s,\cdot)\;\forall i \quad\mbox{and}\quad\Delta_k(s) := Q_k^i(s,\cdot)^T-Q_k^{-i}(s,\cdot),
\ee
where we view $Q_k^i(s,\cdot)\in \mathbb{R}^{|\Action^i|\times |\Action^{-i}|}$ as a $|\Action^i|\times |\Action^{-i}|$-dimensional matrix. The high-order term $O(\alpha_k^2)$ is due to \Cref{assume:step}. 

In the policy update \eqref{eq:piupdate}, the actor targets the $\epsilon$-greedy response in which each action $a^i$ gets played with at least $\epsilon/|\Action^i|>0$ probability all the time. Since the actor takes a convex combination of the target and the previous policy, each action always gets played with at least $\epsilon/|\Action^i|>0$ according to the policies learned. Thus, we have
\be\label{eq:nonnegative}
(\hatpi_k^{i}(s) - \pi_k^i(s))^Tq_k^i(s,\cdot)\geq 0\quad\mbox{and}\quad (\hatpi_k^{-i}(s) - \pi_k^{-i}(s))^Tq_k^{-i}(s,\cdot)\geq 0
\ee
for all $k$ and $s$. This non-negativity plays an important role in showing quasi-monotonicity \eqref{eq:liminfYcondition} without characterizing the decay rate of the Nash gap \eqref{eq:gap}. It is instructive to note that \eqref{eq:nonnegative} does not necessarily hold for smoothed best responses.

Combining \eqref{eq:Ydiff}, \eqref{eq:udiff2}, and \eqref{eq:nonnegative}, we can show that
\begin{flalign}
Y_{k+1}^i(s,a) \geq (1-\beta_k)Y_k^i(s,a) &+ \gamma \beta_k \uY_k^i + e_k^i(s,a) - z_{k+1}^i(s,a) + z_k^i(s,a),\label{eq:YY}
\end{flalign}
where we define
\begin{subequations}\label{eq:ez}
\begin{flalign}
&\uY_k^i := \min_{(s',a')} Y_k^i(s',a'),\label{eq:uYdef}\\
&e_k^i(s,a) := \gamma \sum_{s'} p(s'\mid s,a) \Big(\alpha_k(\hatpi_k^i(s') - \pi_k^i(s'))^T\delta_k^i(s') \nn\\
&\hspace{.6in}+ \alpha_k(\hatpi_k^{-i}(s') - \pi_k^{-i}(s'))^T\big(\delta_k^{-i}(s') + \Delta_k(s')\pi_k^i(s')\big) + O(\alpha_k^2)\Big),\\
&z_k^i(s,a) := \gamma \sum_{s'} p(s'\mid s,a) \pi_k^i(s')^T\delta_k^i(s').
\end{flalign}
\end{subequations}
Iterating \eqref{eq:YY} from $0$ to $k$ and telescoping the terms $- z_{k+1}^i(s,a) + z_k^i(s,a)$ yield that
\begin{flalign*}
Y_{k+1}^i(s,a)\geq \frac{\gamma}{k+1}\sum_{l=0}^k \uY_l^i - z_{k+1}^i(s,a) +\frac{1}{k+1}\sum_{l=0}^kz_l^i(s,a) + \frac{1}{k+1}\sum_{l=0}^k (l+1)e_l^i(s,a).
\end{flalign*}
Define
\be\label{eq:eta}
\eta_k^i := \min_{(s,a)}\left\{- z_{k}^i(s,a) +\frac{1}{k}\sum_{l=0}^{k-1}z_l^i(s,a) + \frac{1}{k}\sum_{l=0}^{k-1} (l+1)e_l^i(s,a)\right\}
\ee
Then, we have
\be\label{eq:uY}
\uY_{k+1}^i \geq \frac{\gamma}{k+1}\sum_{l=0}^k \uY_l^i + \eta_{k+1}^i
\ee
and \eqref{eq:uYdef} implies that
\be\label{eq:uYsum}
\liminf_{k\rightarrow \infty} \inf_{k_2\geq k_1\geq k} \sum_{l =k_1}^{k_2-1} \beta_l \cdot Y_l^i(s,a) \geq \liminf_{k\rightarrow \infty} \inf_{k_2\geq k_1\geq k} \sum_{l =k_1}^{k_2-1} \beta_l \cdot \uY_l^i. 
\ee
Thus, we can complete the proof by showing the non-negativity of the right-hand side based on its evolution \eqref{eq:uY}. However, the long run behavior of the error term $\eta_k^i$ poses a challenge.

\subsubsection{Rate Characterization}
To restrain the impact of the error term in \eqref{eq:uY}, we first characterize the decay rates for the tracking error $\delta_k^i$ and the mismatch error $\Delta_k$, as described in \eqref{eq:deltas}.

\begin{proposition*}\label{prop:rate}
Under \Cref{assume:step,assume:reach}, we have 
\be
\|\delta_k^i(s)\|\in O((k+1)^{\rho_\lambda-0.5})\quad\mbox{and}\quad \|\Delta_k^i(s)\| \in O((k+1)^{-1})
\ee 
for all $s$ and $i$ almost surely.
\end{proposition*}

\begin{proof}
We first focus on the decay rate of $\delta_k^i(s,a)$. For notational simplicity, let
\be
y_k^i(s,a^i):= \frac{\ind{(s,a^i)=(s_k^{},a_k^i)}}{\pi_k^i(a^i\mid s)}\cdot (\hatq_k^i - q_k^i(s,a^i))
\ee
such that $q_{k+1}^i \equiv q_k^i+\lambda_k y_k$. Then, \eqref{eq:update} and \eqref{eq:deltas} yield
\begin{flalign}
\|\delta_{k+1}^i(s)\|_2^2 = &\;\|\delta_k^i(s)\|_2^2 - 2 \lambda_k \delta_k^i(s)^Ty_k^i(s,\cdot) + 2\alpha_k \delta_k^i(s)^TQ_k^i(s,\cdot)(\hatpi_k^{-i}(s)-\pi_k^{-i}(s)) \nn\\
&+ 2\beta_k \delta_k^i(s)^TY_k^i(s,\cdot)\pi_k^{-i}(s) + O(\lambda_k^2).\label{eq:deltanorm}
\end{flalign}
Define the filtration $\mathcal{F}_k:=\sigma(q_l^i,\pi_l^i,v_l^i,s_{l-1},a_{l-1};l\leq k, i\in [n])$ and 
\begin{subequations}
\begin{flalign}
&\bary_k^i(s,a^i):=\E[y_k^i(s,a^i)\mid \mathcal{F}_k] = p(s\mid s_{k-1},a_{k-1})\delta_k^i(s)\\
&\omega_k^i(s,a^i) := y_k^i(s,a^i) - \bary_k^i(s,a^i),
\end{flalign}
\end{subequations}
which implies that $\E[\omega_k^i(s,a^i)\mid \mathcal{F}_k] = 0$. Then, we can write \eqref{eq:deltanorm} as
\begin{flalign}
\|\delta_{k+1}^i(s)\|_2^2 = &\;(1-2\lambda_kp(s\mid s_{k-1},a_{k-1}))\|\delta_k^i(s)\|_2^2  + 2\beta_k \delta_k^i(s)^TY_k^i(s,\cdot)\pi_k^{-i}(s) \nn\\ &
+ 2\alpha_k \delta_k^i(s)^TQ_k^i(s,\cdot)(\hatpi_k^{-i}(s)-\pi_k^{-i}(s)) -2\lambda_k \delta_k^i(s)^T\omega_k^i(s,\cdot)+ O(\lambda_k^2).\label{eq:deltanorm2}
\end{flalign}

Based on the Cauchy-Schwartz and AM-GM inequalities, we have
\begin{subequations}\label{eq:CS}
\begin{flalign}
&2\delta_k^i(s)^TY_k^i(s,\cdot)\pi_k^{-i}(s) \leq d_k \|Y_k^i(s,\cdot)\pi_k^{-i}(s)\|_2^2 + \frac{1}{d_k} \|\delta_k^i(s)\|_2^2\\
&2\delta_k^i(s)^TQ_k^i(s,\cdot)(\hatpi_k^{-i}(s)-\pi_k^{-i}(s)) \leq c_k \|Q_k^i(s,\cdot)(\hatpi_k^{-i}(s)-\pi_k^{-i}(s))\|_2^2 + \frac{1}{c_k}\|\delta_k^i(s)\|_2^2
\end{flalign}
\end{subequations}
for any $d_k,c_k>0$. By \Cref{assume:reach}, there exists $p_o >0$ such that $p_o \leq p(s'\mid s,a)$ for all $(s,a,s')$. Furthermore, the boundedness of the iterates due to the nature of the update \eqref{eq:update} and \Cref{assume:step} yields that there exists $C_o > 0$ such that $C_o \geq \|Q_k^i(s,\cdot)(\hatpi_k^{-i}-\pi_k^{-i}(s))\|_2^2$ and $C_o \geq \|Y_k^i(s,\cdot)\pi_k^{-i}(s)\|_2^2 $.
Let $\tlambda_k:=p_o \lambda_k$. Then, by \eqref{eq:deltanorm} and \eqref{eq:CS}, we obtain
\begin{flalign}\nn
\|\delta_{k+1}^i(s)\|_2^2 \leq &\;\bigg(1-2\tlambda_k + \frac{\alpha_k}{c_k} + \frac{\beta_k}{d_k}\bigg)\|\delta_k^i(s)\|_2^2 + (\alpha_k c_k + \beta_k d_k) C_o + \tlambda_k \xi_k^i(s) + O(\tlambda_k^2),
\end{flalign}
where $\xi_k^i(s):= -2\delta_k^i(s)^T\omega_k^i(s,\cdot)/p_o$ and $\E[\xi_k^i(s)\mid\mathcal{F}_k] = 0$ since $\delta_k^i$ is $\mathcal{F}_k$-measurable. Set $c_k = \frac{\alpha_k}{2\tlambda_k}$ and $d_k = \frac{\beta_k}{2\tlambda_k}$. Then, we can bound it by
\begin{flalign}\nn
\|\delta_{k+1}^i(s)\|_2^2 \leq &\;(1-\tlambda_k)\|\delta_k^i(s)\|_2^2 + \tlambda_k\bigg(\frac{\alpha_k^2}{\tlambda_k^2} C_o + \frac{\beta_k^2}{\tlambda_k^2} C_o + \tlambda_k D_o + \xi_k^i(s)\bigg),
\end{flalign}
for some constant $D_o>0$.

Fix state $s$ and agent $i$. Since $\rho_\lambda\leq 2\rho_\alpha-2\rho_\lambda< 2\rho_\beta-2\rho_\lambda$ by \Cref{assume:step,assume:additional}, we introduce an auxiliary sequence $\{\theta_k\}_{k\geq k_o}$ evolving according to
\be
\theta_{k+1} \leq (1-\tlambda_k)\theta_k + \tlambda_k\bigg(\frac{C_1}{(k+1)^{\rho_\lambda}} + \xi_k^i(s)\bigg)
\ee
for all $k\geq k_o$ and $\theta_{k_o} = \|\delta_{k_o}^i(s)\|_2^2$ for some $C_1>0$ such that $\|\delta_k^i(s)\|_2^2 \leq \theta_k$ for all $k\geq k_o$. Since $\tlambda_k\in\Theta((k+1)^{-\rho_\lambda})$, and $\rho_\lambda\in(0.5,1)$, for this non-negative sequence, we can invoke \citep[Lemma 1]{ref:Liu22} to conclude $\theta_k\in O((k+1)^{2\rho_\lambda-1})$ almost surely. Since $0\leq \|\delta_k^i(s)\|_2^2 \leq \theta_k$, we have $\|\delta_k^i(s)\|_2 \in O((k+1)^{\rho_\lambda-0.5})$ almost surely.

Next, we address the decay rate of $\Delta_k$. By \eqref{eq:update} and \eqref{eq:deltas}, we have
\begin{flalign}
\Delta_{k+1}(s) = &\;(1-\beta_k)\Delta_k(s) + \beta_k \gamma \sum_{s'}p(s'\mid s,a) \Big(\pi_k^i(s')^T\Delta_k(s')\pi_k^{-i}(s') \nn\\
&+ \delta_k^i(s')^T\pi_k^i(s') - \delta_k^{-i}(s')^T\pi_k^{-i}(s')\Big).
\end{flalign}
Fix state $s$. Then, we introduce auxiliary sequences $\{\oDelta_k\}_{k\geq k_o}$ and $\{\uDelta_k\}_{k\geq k_o}$ evolving according to
\begin{flalign*}
\E[\oDelta_{k+1}\mid \mathcal{F}_k] \leq (1-(1-\gamma)\beta_k)\oDelta_k + \beta_k \frac{C_2}{(k+1)^{\rho_\lambda-0.5}} \\
\E[\uDelta_{k+1}\mid \mathcal{F}_k] \leq (1-(1-\gamma)\beta_k)\uDelta_k + \beta_k \frac{C_2}{(k+1)^{\rho_\lambda-0.5}} 
\end{flalign*}
for all $k\geq k_o$ and $\oDelta_{k_o} = \max\{\Delta_{k_o},0\}$, $\uDelta_{k_o}=\max\{-\Delta_{k_o},0\}$ for some $C_2$ such that $-\uDelta_k \leq \Delta_k(s) \leq \oDelta_k$ for all $k\geq k_o$ since $\|\delta_k^i(s)\|_2 \in O((k+1)^{\rho_\lambda-0.5})$. Since $\beta_k = 1/(k+1)$, for these non-negative sequences, we can again invoke \citep[Lemma 1]{ref:Liu22} to conclude $\|\Delta_k(s)\|\in O((k+1)^{-1})$ almost surely.
\end{proof}

Based on \eqref{eq:ez}, \Cref{prop:rate} yields that the decay rate for $e_k^i$ is at most the minimum of $\{\rho_\alpha+\rho_\lambda - 0.5, 1+\rho_\alpha,2\rho_\alpha\}$. Thus, we have
\be
e_k^i(s,a) \in O((k+1)^{0.5-\rho_\alpha-\rho_\lambda}).
\ee
Since $0.5-\rho_\alpha-\rho_\lambda<-1$ by \Cref{assume:additional}, we also have
\be
\frac{1}{k}\sum_{l=0}^{k-1} (l+1)|e_l^i(s,a)| \in O((k+1)^{1.5-\rho_\alpha-\rho_\lambda}).
\ee
On the other hand, the decay rates for $z_k^i$ and its Cesaro mean are given by
\be
z_k^i(s,a) \in O((k+1)^{0.5-\rho_\lambda})\quad\mbox{and}\quad \frac{1}{k}\sum_{l=0}^{k-1} |z_l^i(s,a)| \in O((k+1)^{0.5-\rho_\lambda}).
\ee
Then, by \eqref{eq:eta}, we have $\eta_k^i \in O((k+1)^{1.5-\rho_\alpha-\rho_\lambda})$ since $1.5-\rho_\alpha-\rho_\lambda \geq 0.5-\rho_\lambda$ when $\alpha_k\leq 1$.

\subsubsection{Lower Bound Formulation} We next bound the right-hand side of \eqref{eq:uYsum} based on the recursion \eqref{eq:uY} and the error decay rate $\eta_k^i \in O((k+1)^{1.5-\rho_\alpha-\rho_\lambda})$. Let $\oY_k^i$ denote the Cesaro mean of $\uY_k^i$ as
$\oY_k^i:= \frac{1}{k+1}\sum_{l=0}^k \uY_l^i$.
Then, by \eqref{eq:uY}, we have
\be\label{eq:ooY}
\oY_{k+1}^i \geq \bigg(1-\frac{1-\gamma}{k+2}\bigg)\oY_k^i + \frac{\eta_{k+1}^i}{k+2}.
\ee
The following proposition formulates a decaying lower bound on the Cesaro mean.


\begin{proposition*}
There exists a sequence $\theta_k\in O(k^{-\min\{1-\gamma,\rho_\alpha+\rho_\lambda-1.5\}} \log k)$ bounding the Cesaro mean $\oY_k$ from below, i.e., $\theta_k\leq x_k$ for all $k$.
\end{proposition*}

\begin{proof}
Define $\rho:=\rho_\alpha+\rho_\beta - 1.5$ for notational simplicity and $ \rho\in (0,0.5)$ by \Cref{assume:additional}. Then, there exists some $B>0$ such that $|\eta_{k+1}^i|\leq B(k+2)^{-\rho}$. Therefore, we can bound \eqref{eq:ooY} from below by
\be
\oY_{k+1}^i \geq \bigg(1-\frac{1-\gamma}{k+2}\bigg)\oY_k^i - \frac{B}{(k+2)^{\rho+1}}.
\ee

Consider the auxiliary sequence $\theta_k$ evolving by
\be\label{eq:thetaB}
\theta_{k+1} = \bigg(1-\frac{1-\gamma}{k+2}\bigg)\theta_k - \frac{B}{(k+2)^{\rho+1}}\quad\forall k\geq 0\quad\mbox{and}\quad \theta_0 = \oY_0^i.
\ee
Then, we have $\theta_k\leq \oY_k^i$ for all $k$. Let
\be
\Lambda_{k,l}= \prod_{t=l}^{k-1}\left(1-\frac{1-\gamma}{t+2}\right)\quad\mbox{for } k>l\quad\mbox{and}\quad \Lambda_{k,k}=1.
\ee
Unrolling \eqref{eq:thetaB} gives
\be\label{eq:thetaC}
\theta_k = \Lambda_{k,0}\theta_0 - \sum_{l=0}^{k-1}\Lambda_{k,l+1}\frac{B}{(l+2)^{\rho+1}}.
\ee
Taking logarithm of $\Lambda_{k,l}$ and comparing $\sum\log(1-u)$ to $-\sum u$ and $\int u$ for $u\in(0,1)$, 
we can show that $0 \leq \Lambda_{k,l}\leq \left(\frac{l+2}{k+2}\right)^{1-\gamma}$ for all $k\geq l$. Therefore, by \eqref{eq:thetaC}, we obtain
\begin{flalign}
\theta_k 
&\geq -\frac{c}{(k+2)^{1-\gamma}} \left(1 + \sum_{l=0}^{k-1}\frac{1}{(l+2)^{\rho+\gamma}}\right)
\geq -\frac{c}{(k+2)^{1-\gamma}} \left(1+\int_0^k\frac{dl}{(l+1)^{\rho+\gamma}}\right)\nn\\
&\geq\left\{\begin{array}{ll}
-\frac{c}{1-(\rho+\gamma)} \frac{1}{(k+2)^{\rho}}&\mbox{if } \rho+\gamma < 1\\
-c \frac{1+\log(k+1)}{(k+2)^{1-\gamma}}&\mbox{if }\rho+\gamma = 1\\
-\frac{c(\rho+\gamma)}{(\rho+\gamma)-1}\frac{1}{(k+2)^{1-\gamma}}&\mbox{if } \rho+\gamma > 1
\end{array}\right.
\end{flalign}
for some constant $c\geq\max\{|\theta_0|2^{1-\gamma},B(3/2)^{1-\gamma}\}>0$. This completes the proof.
\end{proof}

By \eqref{eq:uY}, we have $\uY_{k+1}^i \geq \gamma \oY_k^i + \eta_{k+1}^i$ while $\oY_k^i\in O(k^{-\min\{1-\gamma,\rho_\alpha+\rho_\lambda-1.5\}} \log k)$ and $\eta_k^i \in O(k^{-(\rho_\alpha+\rho_\lambda-1.5)})$. This implies that there exists some constant $C>0$ such that
\be
\uY_{k}^i\geq -C \frac{\log(k+1)}{(k+1)^{a}},
\ee
where $a := \min\{1-\gamma,\rho_\alpha+\rho_\lambda-1.5\}$ is positive by \Cref{assume:additional} and since $\gamma\in[0,1)$.
Correspondingly, we have
\begin{flalign}
\sum_{k=k_1}^{k_2-1}\frac{\uY_{k}^i}{k+1} \geq - C \sum_{k=k_1}^{k_2-1}\frac{\log(k+1)}{(k+1)^{1+a}} \geq - C \sum_{k=k_1}^{\infty}\frac{\log(k+1)}{(k+1)^{1+a}}.
\end{flalign}
As $a>0$, the series $\sum_{k=k_1}^{\infty}\frac{\log(k+1)}{(k+1)^{1+a}}<\infty$ is convergent. Therefore, as $k_1\rightarrow\infty$, its tail goes to zero and we obtain
\be
\liminf_{k\rightarrow \infty} \inf_{k_2\geq k_1\geq k} \sum_{l =k_1}^{k_2-1} \frac{Y_l^i(s,a)}{l+1} \geq \liminf_{k\rightarrow \infty} \inf_{k_2\geq k_1\geq k} \sum_{l =k_1}^{k_2-1} \frac{\uY_l^i}{l+1}\geq 0.
\ee
This yields that the global Q-values are almost surely convergent. 

\subsubsection{Limit Characterization}
 Since the global Q-values are almost surely convergent, let $Q_*^i(s,a)$ denote the limit, i.e., $Q_k^i(s,a)\rightarrow Q_*^i(s,a)$ almost surely for each $(i,s,a)$. Then, we can write the local q-function estimate as
\begin{flalign}
q_k^i(s,\cdot) &= Q_*^i(s,\cdot)\pi_k^{-i}(s) + (Q_k^i(s,\cdot)-Q_*^i(s,\cdot))\pi_k^{-i}(s) + \delta_k^i(s)\\
&= \overline{Q}_*^i(s,\cdot)\mu_k^{-i}(s) + (Q_k^i(s,\cdot)-Q_*^i(s,\cdot))\pi_k^{-i}(s) + \delta_k^i(s),
\end{flalign}
where the effective global Q-function in the limit is defined by
\be
 \overline{Q}_*^i(s,\cdot) := (1-\epsilon) Q_*^i(s,\cdot) + \frac{\epsilon}{|\Action^{-i}|}Q_*^i(s,\cdot)\mathbf{1}\mathbf{1}^T
\ee
since the exploration-perturbed actor is given by $\pi_k^{i}(s) = (1-\epsilon)\mu_k^i(s) + \epsilon \frac{\mathbf{1}}{|\Action^i|}$. Therefore, the evolution of the exploration-free actor, as described in \eqref{eq:actormu}, can be written in the generalized-fictitious-play structure \citep{ref:Leslie06} as
\be\label{eq:mu}
\mu_{k+1}^i(s) - (1-\alpha_k)\cdot \mu_k^i(s) \in \alpha_k\cdot \overline{\mathrm{BR}}_k(\mu_k^{-i}(s)), 
\ee
where
\begin{flalign}
\overline{\mathrm{BR}}_k(\mu^{-i}) := \{\mu^i\in\Delta(\Action^i): &\;(\mu^i)^T\overline{Q}_*^i(s,\cdot)\mu^{-i} \geq (\tilde{\mu}^i)^T\overline{Q}_*^i(s,\cdot)\mu^{-i} - \overline{z}_k^i(s),\nn\\
&\forall \tilde{\mu}^i\in\Delta(\Action^i)\}
\end{flalign}
for some $\overline{z}_k^i(s)$ decaying to zero almost surely. Different from \eqref{eq:muZS}, here we have $\overline{Q}_*^i$ rather than $\overline{Q}_k^i$.

By \citep[Theorem 4]{ref:Leslie06}, the limit set of \eqref{eq:mu} is contained in the connected internally chain-recurrent set of the differential inclusion
\be\label{eq:mucont}
\dot{\mu}^i(s) + \mu^i(s) \in \argmax_{\mu^i\in\Delta(\Action^i)} \{(\mu^i)^T\overline{Q}_*^i(s,\cdot)\mu^{-i}\}, 
\ee
which is the continuous-time best response dynamics in an identical-interest game with the payoff $Q_*^i(s,\cdot)=Q_*^{-i}(s,\cdot)$. Therefore, \citep[Theorem 5.5 and Remark 5.6]{ref:Benaim05} yield that $(\mu^i_k(s))_{i\in [n]}$ converge to an equilibrium of the game characterized by $(\Action^i,Q_*^i(s,\cdot))_{i\in[n]}$ almost surely. Given the convergence of the exploration-free actor, the exploration-perturbed actor is also convergent and, therefore, the value function estimates converge to the values of these actor strategies. Finally, we obtain \eqref{eq:NGP} by \Cref{prop:NG}. This completes the proof of \Cref{thm:mainII}.

\section{Illustrative Examples}\label{sec:examples}
In this section, we provide numerical simulations demonstrating the performance of the ADC dynamics and compare its performance with the other state-of-the-art algorithms in the literature. In particular, we compare \Cref{alg:ADC} with the decentralized Q-learning dynamics (Dec-Q) \citep{ref:Sayin21} in two-agent zero-sum stochastic games and with the independent and decentralized learning dynamics (IndDec) \citep{ref:Maheshwari22} in multi-agent identical-interest stochastic games.

\begin{figure}[t]
  \centering
  \begin{subfigure}{0.48\textwidth}
      \centering
      \includegraphics[width=\linewidth]{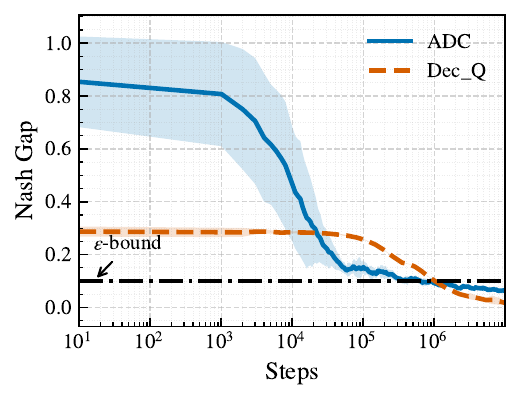}
      \caption{Zero-sum stochastic games.}
      \label{fig:left}
  \end{subfigure}
  \hfill
  \begin{subfigure}{0.48\textwidth}
      \centering
      \includegraphics[width=\linewidth]{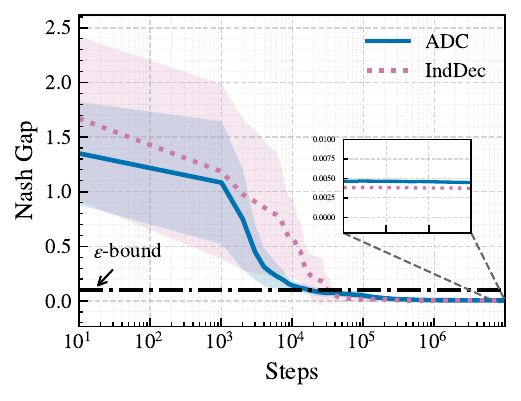}
      \caption{Identical-interest stochastic games.}
      \label{fig:right}
  \end{subfigure}
  \caption{The Nash-gap decay for our ADC, the Dec-Q, and the IndDec dynamics. The dashed line is the $\varepsilon$-threshold \eqref{eq:threshold}. The Nash gap for \Cref{alg:ADC} decays effectively below this threshold for both zero-sum and identical-interest stochastic games.}
  \label{fig:experiment}
\end{figure}

For two-agent zero-sum stochastic games, we consider three states and three actions for each agent. The rewards are generated randomly and bounded by $[0,1]$. The discount factor is $\gamma = 0.8$. For the ADC dynamics, we set step sizes $\lambda_k = (k+1)^{-0.60}$, $\alpha_k = (k+1)^{-0.95}$, $\beta_k = \min(1,10(k+1)^{-1})\in O((k+1)^{-1})$, and an exploration parameter $\epsilon=0.002$. For the Dec-Q dynamics, the step sizes for Q-function and value-function updates are set $\alpha_k = (k+1)^{-0.95}$ and $\beta_k = (k+1)^{-1}$, respectively, with a temperature parameter decaying by $\tau_k =  \tau_o/k + \overline{\tau} (1-1/k)$. Here, the initial temperature parameter $\tau_o = 10^4$ and final temperature parameter is $\overline{\tau}=2 \times 10^{-4}$. 

For multi-agent identical-interest stochastic games, we consider three states and three agents, with each agent having three actions. The rewards are again generated randomly and bounded by $[0,1]$, and the discount factor is $\gamma = 0.8$ . For the ADC dynamics, we use the same step sizes with the zero-sum case. For the IndDec dynamics, we set step sizes $\alpha_k = (k+1)^{-0.60}$ for the Q-function update, and $\beta_k = (k+1)^{-1}$ for the value-function update. Both algorithms share an exploration parameter $\epsilon=0.002$.

We conducted the experiments over $10^7$ time steps, averaging the results across $10$ independent trials. \Cref{fig:experiment} illustrates the mean and standard deviation of the Nash gap, as described in \eqref{eq:gap}. Since there is no actor in the Dec-Q dynamics, we calculate the Nash gap based on the agents' average play. In all these dynamics, the Nash gap decreases over time, indicating convergence to equilibrium. Furthermore, the ADC algorithm demonstrates comparable performance to the state-of-the-art Dec-Q and IndDec algorithms, with the Nash gap decaying below the $\varepsilon$-threshold in both zero-sum and identical-interest games, as shown in \Cref{thm:mainZS,thm:mainII}.

\section{Concluding Remarks}\label{sec:conclusion}
We introduced a new Actor-Dual-Critic learning framework for stochastic games. In this framework, inspired by dual-process cognition, agents maintain a fast critic for immediate payoff response and a slow critic for long-term value approximation. 
We proved that this independent and payoff-based learning dynamic converges to approximate Nash equilibria in both two-agent zero-sum and multi-agent identical-interest stochastic games over an infinite horizon. Our analysis further characterizes the approximation error in terms of the exploration rate. We further conducted numerical experiments demonstrating the effective convergence of the framework across both game classes. Future research could explore extending this dual-critic architecture to broader classes of general-sum games or integrating function approximation to handle large-scale state spaces.

\end{spacing}
\newpage
\begin{spacing}{1}
\bibliographystyle{plainnat}
\bibliography{mybib}
\end{spacing}

\end{document}